%% file: neurips_2026.tex
\documentclass{article}

\PassOptionsToPackage{numbers, compress}{natbib}
\usepackage[preprint]{neurips_2026}
\usepackage{wrapfig}
\usepackage{pifont}
\newcommand{\cmark}{\ding{51}}
\newcommand{\xmark}{\ding{55}}
\usepackage{subcaption}
\usepackage{booktabs}
\usepackage[table]{xcolor}
\usepackage{multirow}
\usepackage[utf8]{inputenc} % allow utf-8 input
\usepackage[T1]{fontenc}    % use 8-bit T1 fonts
\usepackage{hyperref}       % hyperlinks
\usepackage{url}            % simple URL typesetting
\usepackage{booktabs}       % professional-quality tables
\usepackage{amsfonts}       % blackboard math symbols
\usepackage{nicefrac}       % compact symbols for 1/2, etc.
\usepackage{microtype}      % microtypography
\usepackage{xcolor}         % colors
\usepackage{amsmath}
\newtheorem{definition}{Definition}

\newtheorem{proof}{Proof}

\newtheorem{proposition}{Proposition}
\usepackage{multirow}
\usepackage{graphicx}
\usepackage{algorithm}
\usepackage{algpseudocode}
\title{Beyond Factor Aggregation: Gauge-Aware Low-Rank Server Representations for Federated LoRA}

\author{%
  Jinqian Chen \\
  School of Software Engineering\\
  Xi'an Jiaotong University\\
  Xi'an, China 710000 \\
  \texttt{chenjinqian@stu.xjtu.edu.cn} \\
  \And
  Chang Liu \\
  School of Software Engineering\\
  Xi'an Jiaotong University\\
  Xi'an, China 710000 \\
  \texttt{chang.liu@stu.xjtu.edu.cn} \\
  \And
  Jihua Zhu\thanks{Corresponding author.}\\
  School of Software Engineering\\
  Xi'an Jiaotong University\\
  Xi'an, China 710000 \\
  \texttt{zhujh@xjtu.edu.cn} \\
}

\begin{document}

\maketitle

\begin{abstract}
Federated LoRA enables parameter-efficient adaptation of large language models under decentralized data and limited client resources.However, directly averaging LoRA factors is representation-dependent: the same intrinsic update admits infinitely many gauge-equivalent factorizations, so factor-level aggregation can change under arbitrary coordinate choices while the underlying update remains unchanged. This reveals a semantic mismatch in existing federated LoRA aggregation rules. We propose \textbf{GLoRA}, a gauge-aware server representation for federated LoRA.Instead of aggregating raw factors, GLoRA estimates a consensus update subspace from client projectors and aggregates client updates in shared reference coordinates, thereby representing semantic update aggregation entirely in low-rank form. To support heterogeneous client capacities, GLoRA further provides a rank-compatible readout that instantiates adapters of different ranks from the same server state without dense update reconstruction. Experiments on GLUE and SuperNI show that GLoRA consistently outperforms federated LoRA baselines under data, resource, and task heterogeneity, including heterogeneous client ranks, sparse participation, larger backbones, and unseen-task evaluation. GLoRA also achieves a favorable efficiency--performance trade-off, suggesting that effective federated LoRA requires not merely averaging low-rank factors, but defining a semantically meaningful server-side representation for aggregation.
\end{abstract}

\input{sections/1_introduction}
\input{sections/2_related_work}

\input{sections/3_motivation}

\input{sections/4_method}
\input{sections/5_experiments}

% \section{Technical appendices and supplementary material}
% Technical appendices with additional results, figures, graphs, and proofs may be submitted with the paper submission before the full submission deadline (see above). You can upload a ZIP file for videos or code, but do not upload a separate PDF file for the appendix. There is no page limit for the technical appendices. 

% Note: Think of the appendix as ``optional reading'' for reviewers. The paper must be able to stand alone without the appendix; for example, adding critical experiments that support the main claims to an appendix is inappropriate. 

%%%%%%%%%%%%%%%%%%%%%%%%%%%%%%%%%%%%%%%%%%%%%%%%%%%%%%%%%%%%

\newpage

\bibliographystyle{plainnat}
\bibliography{reference}

\input{sections/appendix}
% \newpage
% \input{checklist}
\end{document}

%% file: sections/1_introduction.tex
\section{Introduction}

Federated fine-tuning~\cite{FedAvg, FedIT} has emerged as an increasingly important paradigm for adapting foundation models under decentralized data and client-side resource constraints.
In this setting, parameter-efficient fine-tuning (PEFT)~\cite{LoRA, DoRA, AdaLoRA} is especially attractive, since only a small fraction of model parameters needs to be updated and communicated.
Among PEFT methods, LoRA~\cite{LoRA} is particularly compelling: it represents adaptation as a low-rank update and therefore offers the promise of both lightweight local training and lightweight server aggregation.
This makes federated LoRA a natural candidate for scaling LLM adaptation across heterogeneous clients.

Despite this promise, federated LoRA exposes a fundamental mismatch between \emph{what is cheap to communicate} and \emph{what is meaningful to aggregate}.
A LoRA update is represented as $\Delta W = BA$, but this factorization is not unique:
the same update matrix admits infinitely many gauge-equivalent realizations $(B,A)$ and $(BQ,Q^{-1}A)$.
In centralized optimization, this ambiguity is often less exposed because the factors remain internal variables of a single optimizer.
In federated optimization, by contrast, LoRA factors are transmitted across clients and reused as server-side aggregation objects.
At that point, the ambiguity is no longer merely algebraic; it becomes a semantic issue of the server update rule itself.
If the server aggregates raw factors directly, then identical intrinsic client updates may still induce different server outputs under gauge-equivalent reparameterizations.
In other words, the aggregation rule is no longer defined on the update matrices themselves, but on arbitrary coordinate realizations.

A substantial line of recent federated LoRA work~\cite{FFA-LoRA, FedEx-LoRA, FedSA-LoRA, LoRA-Fair} has focused on the \emph{inexact aggregation} problem, namely the mismatch between aggregating products and aggregating factors.
This line of work is important, but it leaves a more basic question unresolved:
before asking whether a factor-level rule recovers the correct aggregate, one must first ask whether raw factors are semantically valid aggregation objects at all.
A straightforward repair for gauge ambiguity in federated LoRA is to aggregate induced updates $\Delta W_i = B_iA_i$ directly and then refactorize the result for redistribution~\cite{FlexLoRA}.
This avoids the semantic defect of raw factor aggregation, but it also forces the server back into dense-update materialization and matrix factorization.
For large transformer layers, this means forming and decomposing matrices in $\mathbb{R}^{d_{\mathrm{out}}\times d_{\mathrm{in}}}$, which is increasingly unattractive at LLM scale and misaligned with the low-rank efficiency premise that motivates federated PEFT in the first place.

This paper is built around the following question:
\emph{Can we preserve the semantic correctness of update-level aggregation without collapsing back to dense-update aggregation?}
We answer this question by shifting the focus from aggregation rules to \emph{server representations}.
Our key idea is that the federated server should neither aggregate raw factor coordinates nor materialize full dense updates.
Instead, it should maintain a \emph{gauge-aware low-rank representation} that captures the geometry of current-round client updates and expresses them under a shared reference frame.

We instantiate this principle in \textbf{GLoRA}, a gauge-aware low-rank server representation for federated LoRA.
GLoRA first rewrites each client update in subspace--coordinate form by applying a gauge fixing step: $B_i=U_iR_i$ and $A'_i=R_iA_i$, where $U_i$ is an orthonormal basis for the client's update subspace.
The server then aggregates the gauge-invariant projectors $U_iU_i^\top$ to estimate a consensus subspace $U_{\mathrm{ref}}$.
Finally, each client update is translated into this shared reference frame and aggregated through comparable coordinates.
The resulting server state $(U_{\mathrm{ref}},A_{\mathrm{global}})$ remains low-rank, yet corresponds exactly to the projection of the dense update average onto the learned consensus subspace.
For heterogeneous clients, GLoRA reads out adapters of different ranks from the same server state, avoiding separate aggregation objects or dense refactorization.  We evaluate GLoRA on GLUE~\cite{GLUE} and SuperNI~\cite{SuperNI} under data heterogeneity, heterogeneous client ranks, sparse participation, larger backbones, and unseen-task evaluation.
Across these settings, GLoRA consistently improves over federated LoRA baselines based on direct factor aggregation and achieves a favorable efficiency--performance trade-off relative to dense-update aggregation methods.
These results suggest that the key design choice in federated LoRA is not merely how to average low-rank factors, but what server-side object should represent the aggregate update.  Contributions are summarized as follows:
\begin{itemize}
    \item We identify \emph{gauge dependence} as a semantic failure mode of factor-level federated LoRA aggregation. Unlike the inexact aggregation problem, which asks whether a factor-level rule recovers a desired product-level aggregate, gauge dependence asks a more basic question: whether raw LoRA factors are representation-invariant aggregation objects at all.

    \item We propose \textbf{GLoRA}, a gauge-aware aggregation method for federated LoRA. 
GLoRA maintains a compact low-rank server state, aggregates client updates through a consensus subspace and shared reference coordinates, and instantiates heterogeneous clients via rank-compatible readout, avoiding both raw factor aggregation and dense-update materialization.

    \item We conduct extensive experiments on GLUE and SuperNI under data, resource, and task heterogeneity with sparse participation. Results show that GLoRA consistently improves over federated LoRA baselines, scales to larger backbones, and achieves a favorable efficiency--performance trade-off compared with dense-update aggregation.
\end{itemize}

%% file: sections/2_related_work.tex
\section{Related Work}

\begin{figure}
\vspace{-1mm}
    \centering
    \includegraphics[width=0.8\linewidth]{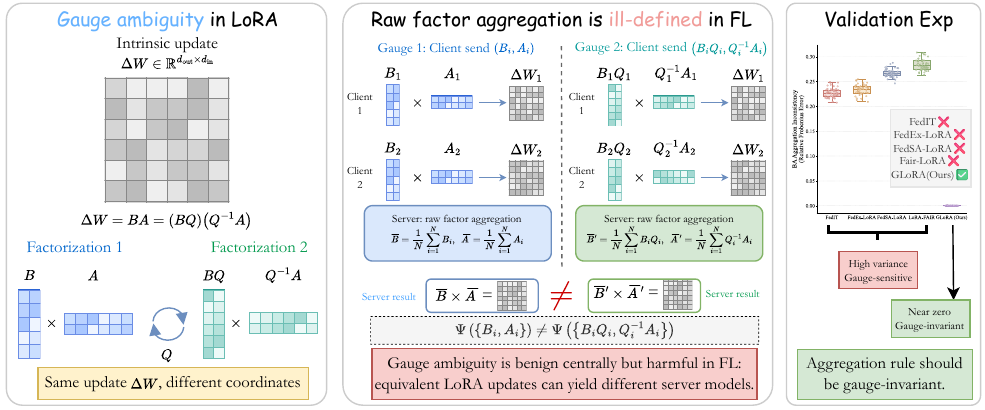}
\caption{
Motivation of gauge-aware LoRA aggregation in FL.
Gauge-equivalent factors, e.g., $(B,A)$ and $(BQ,Q^{-1}A)$, represent the same update
$\Delta W=BA$ but can yield different server models under raw factor aggregation.
The validation experiment shows that existing factor-aggregation methods are gauge-sensitive,
whereas GLoRA achieves near-zero aggregation inconsistency.
}
    \label{fig:motivation}
    \vspace{-4mm}
\end{figure}

\paragraph{Parameter-Efficient Fine-Tuning.}
Parameter-efficient fine-tuning (PEFT) adapts large pre-trained models by updating only a small subset of parameters.
Representative directions include adapter-based tuning~\citep{houlsby2019parameter}, prompt/prefix-based tuning~\citep{li2021prefix,lester2021power}, and multiplicative modulation methods such as IA$^3$~\citep{liu2022few}.
Among them, LoRA~\citep{LoRA} has become particularly influential due to its low-rank parameterization and zero additional inference latency after merging.
Subsequent variants mainly improve LoRA in centralized settings, such as adaptive rank allocation in AdaLoRA~\citep{AdaLoRA}, quantization-aware fine-tuning in QLoRA~\citep{dettmers2023qlora}, and enhanced decomposition in DoRA~\citep{DoRA}.
These methods substantially improve PEFT efficiency and expressiveness, but they do not address the server-side semantic validity of aggregating LoRA updates across federated clients.

\textbf{Federated LoRA.}
Recent work incorporates LoRA into federated fine-tuning of foundation models.
A first line of research directly aggregates LoRA factors or modifies factor-level aggregation rules, including FedIT~\citep{FedIT}, FFA-LoRA~\citep{FFA-LoRA}, FedSA-LoRA~\citep{FedSA-LoRA}, LoRA-FAIR~\citep{LoRA-Fair}, and FedEx-LoRA~\citep{FedEx-LoRA}.
These methods mainly focus on reducing aggregation error, improving robustness, or addressing the algebraic \emph{inexact aggregation} problem.
Our work differs in viewpoint: instead of asking how to aggregate LoRA factors more accurately, we ask a logically prior question---whether raw factor coordinates are semantically valid server-side aggregation objects at all under gauge ambiguity. A second line of work studies heterogeneity in federated LoRA, especially heterogeneous client ranks and system budgets.
HetLoRA~\citep{HetLoRA}, FLoRA~\citep{FLoRA}, and FlexLoRA~\citep{FlexLoRA} design rank-adaptive or resource-aware redistribution mechanisms for heterogeneous clients. These methods highlight the practical importance of client heterogeneity, but they do not directly resolve the representation issue raised in this paper. In contrast, we treat heterogeneous rank not as the starting point of the problem, but as a design constraint that the server-side representation must naturally support.

%% file: sections/3_motivation.tex
\section{Rethinking LoRA in Distributed Optimization}
\label{sec:rethinking}

\subsection{LoRA Is an Equivalence-Class Parameterization}
\label{subsec:lora_equiv}

LoRA represents an adaptation matrix as a low-rank product
$
\Delta W = BA$,
where $B \in \mathbb{R}^{d_{\mathrm{out}} \times r}$ and
$A \in \mathbb{R}^{r \times d_{\mathrm{in}}}$ with
$r \ll \min(d_{\mathrm{out}}, d_{\mathrm{in}})$.
This representation is efficient but not unique.
For any invertible matrix $Q \in \mathrm{GL}(r)$,
\begin{equation}
BA = (BQ)(Q^{-1}A).
\label{eq:gauge_equiv}
\end{equation}
Hence, a LoRA update is not identified by a unique factor pair, but by an equivalence class of factorizations that induce the same update matrix.

\begin{definition}[Gauge-equivalent LoRA factorizations]
\label{def:gauge_equiv}
Two factor pairs $(B,A)$ and $(\widetilde B,\widetilde A)$ are \emph{gauge-equivalent} if there exists an invertible matrix $Q \in \mathrm{GL}(r)$ such that
$
\widetilde B = BQ,
\widetilde A = Q^{-1}A.
$
Their equivalence class is
$
[(B,A)] \;=\; \{(BQ,Q^{-1}A): Q \in \mathrm{GL}(r)\}.
$
\end{definition}

This ambiguity is often less exposed in centralized training, where the factors remain internal variables of a single optimizer.
In distributed optimization, however, LoRA factors are no longer private coordinates:
they are transmitted across clients and must serve as server-side aggregation objects.
At that point, gauge ambiguity is no longer merely an algebraic nuisance.
It becomes a semantic issue of the server update rule itself.

\subsection{Distributed LoRA Requires a Gauge-Invariant Server Update Rule}
\label{subsec:server_state}

Consider round $t$ with active client set $\mathcal C_t$.
Each client $i \in \mathcal C_t$ produces a local LoRA update
$
\Delta W_i^t = B_i^t A_i^t,
$
where $(B_i^t,A_i^t)$ is only one representative of the equivalence class $[(B_i^t,A_i^t)]$.
A distributed LoRA method may therefore be abstracted as a server update rule
$
\Omega^{t+1}
=
\Psi\!\left(\{(B_i^t,A_i^t)\}_{i\in\mathcal C_t},\, \Omega^t\right),
$
where the central question is whether $\Psi$ depends on the intrinsic client updates $\{\Delta W_i^t\}$, or on the arbitrary factor coordinates used to express them.

\begin{definition}[Gauge-invariant state update]
\label{def:rep_invariance}
A server update rule $\Psi$ is \emph{gauge-invariant} if, for any round $t$ and any collection of invertible matrices $\{Q_i\}_{i\in\mathcal C_t}$,
$
\Psi\!\left(\{(B_i^t,A_i^t)\}_{i\in\mathcal C_t}, \Omega^t\right)
=
\Psi\!\left(\{(B_i^tQ_i,Q_i^{-1}A_i^t)\}_{i\in\mathcal C_t}, \Omega^t\right).
$
\end{definition}

Definition~\ref{def:rep_invariance} formalizes the minimal requirement for distributed LoRA:
the server should respond to the underlying updates themselves, rather than to arbitrary factor coordinates.
This immediately explains why direct factor-space aggregation is problematic.
Any rule that operates explicitly on $(B_i,A_i)$, rather than only on $\Delta W_i = B_iA_i$, can in general change its output under a gauge-equivalent reparameterization, even when the underlying client updates remain fixed (See Fig.~\ref{fig:motivation} panel~3).

\paragraph{Discussion: Gauge invariance vs. inexact aggregation.}
Existing federated LoRA methods often emphasize the \emph{inexact aggregation} issue, i.e.,$
\sum_i B_iA_i \neq \Big(\sum_i B_i\Big)\Big(\sum_i A_i\Big)$,
which asks whether a factor-level rule can recover the desired update average.
We argue that gauge invariance is a more fundamental requirement.
Before asking how to aggregate LoRA factors, one must first ask whether raw factors are valid aggregation objects at all.
Under gauge-equivalent reparameterizations, the same intrinsic update $\Delta W_i=B_iA_i$ can be represented by different coordinates, so aggregating raw factors may yield different server updates for the same client updates.
This is a semantic failure, whereas inexact aggregation is an algebraic approximation error after a representation has already been fixed.

A straightforward remedy is to aggregate update matrices directly $
\Delta W_g^t=\sum_{i\in\mathcal C_t}p_i^t B_i^tA_i^t,
$
and then refactorize the result for redistribution.
While semantically valid, this reduces federated LoRA to dense-update aggregation, requiring the server to materialize and decompose a
$d_{\mathrm{out}}\times d_{\mathrm{in}}$ matrix.
Such a solution undermines the efficiency premise of PEFT at LLM scale.

This motivates our central design question:
\emph{Can we build a server-side low-rank representation that is gauge-invariant, avoids dense re-materialization, and supports heterogeneous client ranks?}
GLoRA answers this question by aggregating client updates through a shared consensus subspace and rank-compatible low-rank coordinates.

%% file: sections/4_method.tex
\section{GLoRA: Gauge-aware LoRA Aggregation for Distributed Optimization}
\label{sec:method}

The discussion above identifies the required server object: it should be defined
by the intrinsic client updates, but represented without dense materialization.
GLoRA realizes this object with a subspace--coordinate server state. At round $t$,
the server maintains $
\Omega^t=(U_{\mathrm{ref}}^t,Z_g^t),
$
where $U_{\mathrm{ref}}^t\in\mathbb{R}^{d_{\mathrm{out}}\times R}$ is an
orthonormal reference basis with rank budget $R$, and
$Z_g^t\in\mathbb{R}^{R\times d_{\mathrm{in}}}$ stores the aggregated coordinates
in this basis. Together they induce the low-rank server update $
\Delta W_g^t=U_{\mathrm{ref}}^tZ_g^t.$ The role of the two components is deliberately separated:
$U_{\mathrm{ref}}^t$ specifies the common update geometry to which clients are
aligned, while $Z_g^t$ stores the aggregated update after this alignment.
The rest of this section explains how GLoRA constructs and redistributes this
state: Sec.~\ref{subsec:gauge_fixing} extracts gauge-invariant update geometry,
Sec.~\ref{subsec:coord_agg} aggregates projected updates in the shared frame,
and Sec.~\ref{subsec:hetero_redist} reads out rank-compatible adapters for
heterogeneous clients.

\subsection{Gauge-Invariant Subspace Extraction and Consensus Geometry}
\label{subsec:gauge_fixing}

Each active client $i\in\mathcal C_t$ returns a LoRA update
$\Delta W_i^t = B_i^t A_i^t, 
B_i^t\in\mathbb{R}^{d_{\mathrm{out}}\times r_i}, 
A_i^t\in\mathbb{R}^{r_i\times d_{\mathrm{in}}},
$
where $r_i$ may vary across clients.
Since $(B_i^t,A_i^t)$ is only one representative of the equivalence class
inducing $\Delta W_i^t$, GLoRA first converts the update into a subspace--coordinate form and extracts its gauge-invariant column-subspace geometry via a reduced QR decomposition\footnote{We assume full column rank for notation. If $B_i^t$ is rank-deficient, replace $r_i$ with $\operatorname{rank}(B_i^t)$; the results still hold.}:
\begin{equation}
B_i^t = U_i^t T_i^t,\qquad
(U_i^t)^\top U_i^t = I_{r_i},\qquad
\widehat A_i^t = T_i^t A_i^t .
\label{eq:qr_decomposition}
\end{equation}
Then
$
\Delta W_i^t = U_i^t \widehat A_i^t .
$
This separates the update into a column subspace $\mathrm{span}(U_i^t)$ and
coordinates $\widehat A_i^t$ inside that subspace. The column subspace is the part of the update geometry that is comparable across
clients. We represent it by the orthogonal projector
$
P_i^t = U_i^t(U_i^t)^\top.
$

Unlike raw factors, $P_i^t$ is invariant to any invertible reparameterization
$B_i^t\mapsto B_i^tQ_i$, because such a transformation does not change the
subspace spanned by $B_i^t$.
The server therefore estimates the shared update geometry by
\begin{equation}
K^t = \sum_{i\in\mathcal C_t} p_i^t P_i^t
    = \sum_{i\in\mathcal C_t} p_i^t U_i^t(U_i^t)^\top ,
    \quad K^t\in\mathbb{R}^{d_{\mathrm{out}}\times d_{\mathrm{out}}}
\label{eq:subspace_covariance}
\end{equation}
where $p_i^t\ge 0$ and $\sum_{i\in\mathcal C_t}p_i^t=1$.
The reference basis is defined as
\begin{equation}
U_{\mathrm{ref}}^t = \mathrm{TopEig}_R(K^t),\qquad
(U_{\mathrm{ref}}^t)^\top U_{\mathrm{ref}}^t=I_R .
\label{eq:consensus_basis}
\end{equation}
Thus, $U_{\mathrm{ref}}^t$ captures the rank-$R$ consensus subspace most
consistently supported by current-round client updates. Importantly, this step does not require materializing the dense matrix
$K^t$.
In practice, the same eigenspace can be obtained from the thin matrix
$
M^t =
\big[\sqrt{p_i^t}U_i^t\big]_{i\in\mathcal C_t},
$
since $K^t=M^t(M^t)^\top$.
The server therefore operates on a matrix with only
$\sum_{i\in\mathcal C_t} r_i$ columns.

\begin{proposition}[Exactness under sufficient server rank.]
\label{prop:exact_sufficient_rank}
Let $\Delta W_i^t=U_i^t\widehat A_i^t$ with $(U_i^t)^\top U_i^t=I_{r_i}$, and let
$
r_{\cup}^t=\mathrm{rank}\big([U_i^t]_{i\in\mathcal C_t}\big)
$
be the dimension of the union span of participating client update subspaces.
If the server rank satisfies $R\ge r_{\cup}^t$, then GLoRA recovers the dense update average exactly:
$
U_{\mathrm{ref}}^t Z_g^t
=
\sum_{i\in\mathcal C_t} p_i^t \Delta W_i^t .
$
In particular, since $r_{\cup}^t\le \sum_{i\in\mathcal C_t}r_i$, the condition
$R\ge \sum_{i\in\mathcal C_t}r_i$ is sufficient for exact aggregation.
The proof is provided in Appendix~\ref{app:exactness_proof}.
\end{proposition}

\subsection{Projected Update Aggregation in Shared Coordinates}
\label{subsec:coord_agg}

After $U_{\mathrm{ref}}^t$ is fixed, each client update can be expressed in the
same reference frame as
$Z_i^t=(U_{\mathrm{ref}}^t)^\top U_i^t\widehat A_i^t
=(U_{\mathrm{ref}}^t)^\top\Delta W_i^t
\in\mathbb{R}^{R\times d_{\mathrm{in}}}$.
Its reconstruction,
$U_{\mathrm{ref}}^tZ_i^t
=U_{\mathrm{ref}}^t(U_{\mathrm{ref}}^t)^\top\Delta W_i^t$,
is exactly the orthogonal projection of client $i$'s update onto the consensus
subspace.

Since all $Z_i^t$ now share the same basis, they can be aggregated directly as
$Z_g^t=\sum_{i\in\mathcal C_t}p_i^tZ_i^t$, yielding the server update
$\Delta W_g^t=U_{\mathrm{ref}}^tZ_g^t$.
Combining the two expressions gives
$Z_g^t=(U_{\mathrm{ref}}^t)^\top\sum_{i\in\mathcal C_t}p_i^t\Delta W_i^t$,
and hence
\begin{equation}
\Delta W_g^t
=
U_{\mathrm{ref}}^t(U_{\mathrm{ref}}^t)^\top
\sum_{i\in\mathcal C_t}p_i^t\Delta W_i^t .
\label{eq:key_identity_proj}
\end{equation}

Equation~\eqref{eq:key_identity_proj} is the key identity of GLoRA.
It shows that GLoRA performs update-level aggregation, but only inside the learned
consensus subspace.
Therefore, the server update is defined by the intrinsic matrices
$\{\Delta W_i^t\}$ rather than by arbitrary LoRA coordinates, while the server
never materializes the dense average $\sum_i p_i^t\Delta W_i^t$.

\begin{proposition}[Gauge invariance of GLoRA aggregation]
\label{prop:gauge_invariance}
For any round $t$, the GLoRA server update $\Delta W_g^t$ is invariant to
client-side gauge reparameterizations. That is, replacing each submitted factor
pair $(B_i^t,A_i^t)$ by $(B_i^tQ_i,Q_i^{-1}A_i^t)$ for any invertible
$Q_i\in\mathrm{GL}(r_i)$ leaves $\Delta W_g^t$ unchanged.
\end{proposition}

\subsection{Client-Aware Readout for Heterogeneous Ranks}
\label{subsec:hetero_redist}

The server maintains a single state
$\Omega^t=(U_{\mathrm{ref}}^t,Z_g^t)$, but client $i$ may only instantiate a
rank-$r_i$ adapter.
Thus, redistribution requires a rank-compatible readout from the same server
update.

We first put the server update into an energy-ordered low-rank form.
Because $U_{\mathrm{ref}}^t$ is orthonormal, it is sufficient to decompose the
small coordinate matrix:
\begin{equation}
Z_g^t = O^t\Sigma^t(V^t)^\top,\qquad
U_s^t = U_{\mathrm{ref}}^tO^t .
\label{eq:spectral_readout}
\end{equation}
Then $\Delta W_g^t=U_s^t\Sigma^t(V^t)^\top
=\sum_{j=1}^{\ell}\sigma_j^t u_j^t(v_j^t)^\top$, where
$\ell=\min(R,d_{\mathrm{in}})$.
A purely spectral readout would simply send the top-$r_i$ spectral components to
client $i$.
This is globally optimal for approximating $\Delta W_g^t$ in Frobenius norm, but
it may be suboptimal under data heterogeneity, where different clients benefit
from different directions.

GLoRA therefore uses a core--tail readout.
The core set $\mathcal G_i^t=\{1,\ldots,\lfloor\gamma r_i\rfloor\}$, with
$\gamma\in[0,1]$, preserves a common shared update.
The remaining rank budget is assigned according to client-specific update
geometry.
Let $H_i^{t-}$ be the latest available orthonormal basis of client $i$'s
historical update subspace.
For each non-core component, we compute the alignment score
$a_{ij}^t=\|(H_i^{t-})^\top u_j^t\|_2^2$.
The selected component set is
$
\mathcal I_i^t
=
\mathcal G_i^t\cup
\mathrm{TopK}_{r_i-|\mathcal G_i^t|}
\{a_{ij}^t:j\notin \mathcal G_i^t\}.
$
If no history is available, GLoRA falls back to the global spectral order. Given $\mathcal I_i^t$, the server initializes client $i$ with the balanced factorization
\begin{equation}
B_{i,\mathrm{init}}^t
=
U_s^t[:,\mathcal I_i^t]\mathrm{diag}\!\left(\sqrt{\sigma_{\mathcal I_i^t}^t}\right),
\qquad
A_{i,\mathrm{init}}^t
=
\mathrm{diag}\!\left(\sqrt{\sigma_{\mathcal I_i^t}^t}\right)
\left(V^t[:,\mathcal I_i^t]\right)^\top .
\label{eq:balanced_readout}
\end{equation}
This yields
$B_{i,\mathrm{init}}^tA_{i,\mathrm{init}}^t
=\sum_{j\in \mathcal I_i^t}\sigma_j^t u_j^t(v_j^t)^\top$.
The square-root split preserves the selected update while balancing magnitude
between the two LoRA factors, which provides a better initialization for
subsequent local optimization. This readout does not introduce client-specific server modules.
All clients receive different rank-limited views of the same gauge-aware server
state.
The shared core maintains global consistency, while the client-aware tail uses
the remaining rank budget to adapt to local update geometry. 

The pseudo-code of GLoRA is provided in the Appendix.

%% file: sections/5_experiments.tex
\section{Experiments}

\subsection{Cross-Device Federated Evaluation Setup}
\label{sec:exp_setup_cross_device}

We evaluate GLoRA under three practical sources of cross-device heterogeneity:
data distribution, client resources, and task semantics.

\begin{wrapfigure}{r}{0.4\textwidth}
    \centering
    \vspace{-30pt}
    \includegraphics[width=0.43\textwidth]{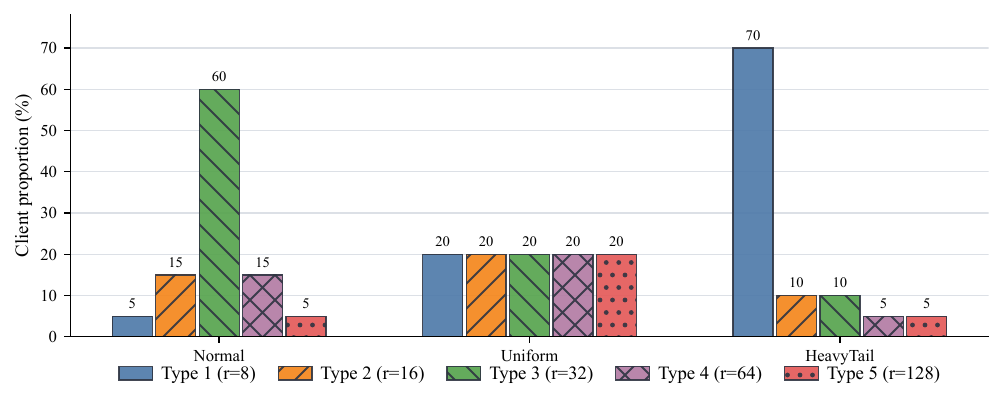}
    \vspace{-10pt}
    \caption{Client types and rank distributions for hetero-rank experiments.}
    \label{fig:rank_distribution}
    \vspace{-25pt}
\end{wrapfigure}

\vspace{-1mm}
\paragraph{Data heterogeneity.}
For GLUE~\cite{GLUE}, we create non-IID client partitions using $\mathrm{Dir}(\alpha)$ with
$\alpha\in\{0.1,0.5\}$, corresponding to highly and moderately heterogeneous
label distributions, respectively. 

\vspace{-1mm}
\paragraph{Resource heterogeneity.}
To simulate resource imbalance across devices, we define five client types with
different LoRA ranks, as shown in Figure~\ref{fig:rank_distribution}. Following
FlexLoRA~\cite{FlexLoRA}, we instantiate three population-level rank patterns: \emph{normal},
\emph{long-tail}, and \emph{uniform}. These settings evaluate whether aggregation
methods can handle heterogeneous adapter capacities.

\vspace{-1mm}
\paragraph{Task heterogeneity.}
We construct a task-heterogeneous benchmark from SuperNI~\cite{SuperNI}. Specifically, we sample
50 task categories from 76 categories and build 50 client groups from 1,600 tasks,
where each client is associated with a distinct task. We additionally hold out
20 unseen and unrelated tasks for unseen-client evaluation, measuring both client performance and generalization ability.

\vspace{-1mm}
\paragraph{Backbone models.}
For GLUE, we follow FedSA-LoRA~\cite{FedSA-LoRA} and FFA-LoRA~\cite{FFA-LoRA} and use RoBERTa-Base~\cite{RoBERTa}. For SuperNI, we
follow FlexLoRA~\cite{FlexLoRA} and use Data-Juicer LLaMA-1B-dj-refine-150B~\cite{JuicerLLama} as the default
backbone. We further evaluate scalability on Qwen2-7B~\cite{Qwen2} and Gemma-2-9B~\cite{Gemma2}. Unless otherwise specified, LoRA fine-tuning is applied only to the query and value projection layers. All experiments were conducted on dual RTX 5090 GPUs, except for the 7B and 9B models, which were run on an single RTX PRO 6000 GPU.

\subsection{Overall Performance with Homogeneous Ranks}
\label{sec:overall_homo_rank}

\paragraph{Implementation details.}
We evaluate
on five GLUE tasks: MNLI-m, MNLI-mm, SST-2, QQP, and QNLI. Following prior federated LoRA studies~\cite{FFA-LoRA,FedEx-LoRA}, for each task, the
training set is split across three clients using Dirichlet label skew with
$\alpha\in\{0.1,0.5\}$. We use full client participation, homogeneous LoRA rank $r=8$, local batch size 128, 10 local steps per round, and 1,000 communication rounds. All methods
are optimized with Adam. For fair comparison, we tune the learning rate over
$\{2\times10^{-4},5\times10^{-4},10^{-3},2\times10^{-3},5\times10^{-3}\}$ and
report the best result for each method.

\paragraph{Main results.}
As shown in Table~\ref{tab:glue_dirichlet}, GLoRA achieves the best average performance under both Dirichlet splits. Under the more challenging $\mathrm{Dir}(0.1)$ setting, it ranks first on four out of five tasks, demonstrating strong robustness to severe client distribution shift, while FedSA fails to converge on most tasks. Under $\mathrm{Dir}(0.5)$, GLoRA obtains the best result on every task. These results suggest that its gains stem not only from handling rank mismatch, but also from aggregating LoRA updates in a well-defined, gauge-aware representation space.

\begin{table*}[t]
\centering
\caption{Performance comparison on GLUE tasks under different Dirichlet splits with homogeneous ranks. The best results are highlighted in \textbf{bold}, and the second-best results are \underline{underlined}.}
\label{tab:glue_dirichlet}
\renewcommand{\arraystretch}{1.15}
\setlength{\tabcolsep}{5.5pt}
\definecolor{oursbg}{HTML}{F4EFFF}
\definecolor{headbg}{HTML}{F7F7F7}
\resizebox{\linewidth}{!}{
\begin{tabular}{llcccccc}
\toprule
\rowcolor{headbg}
\textbf{Split} & \textbf{Method} 
& \textbf{MNLI-m} & \textbf{MNLI-mm} & \textbf{SST-2} 
& \textbf{QQP} & \textbf{QNLI} & \textbf{AVG} \\
\midrule

\multirow{5}{*}{Dir(0.1)}
& FedIT~\cite{FedIT} 
& $71.74 {\scriptstyle \pm 0.45}$ 
& $71.68 {\scriptstyle \pm 0.37}$ 
& $93.69 {\scriptstyle \pm 0.24}$ 
& $83.98 {\scriptstyle \pm 0.37}$ 
& $83.91 {\scriptstyle \pm 0.20}$ 
& $81.00$ \\

& FFA-LoRA~\cite{FFA-LoRA} 
& $67.81 {\scriptstyle \pm 0.39}$ 
& $69.27 {\scriptstyle \pm 0.46}$ 
& $93.92 {\scriptstyle \pm 0.21}$ 
& $82.31 {\scriptstyle \pm 0.33}$ 
& $81.95 {\scriptstyle \pm 0.29}$ 
& $79.05$ \\

& FedEx-LoRA~\cite{FedEx-LoRA} 
& $\underline{71.95 {\scriptstyle \pm 0.48}}$ 
& $\underline{71.97 {\scriptstyle \pm 0.48}}$ 
& $\mathbf{95.27 {\scriptstyle \pm 0.17}}$ 
& $\underline{84.27 {\scriptstyle \pm 0.37}}$ 
& $\underline{84.17 {\scriptstyle \pm 0.35}}$ 
& $\underline{81.33}$ \\

& FedSA-LoRA~\cite{FedSA-LoRA} 
& $53.76 {\scriptstyle \pm 0.93}$ 
& $53.11 {\scriptstyle \pm 1.04}$ 
& $75.00 {\scriptstyle \pm 0.82}$ 
& $63.18 {\scriptstyle \pm 0.90}$ 
& $52.37 {\scriptstyle \pm 0.98}$ 
& $59.48$ \\

\rowcolor{oursbg}
& \textbf{GLoRA (Ours)}
& $\mathbf{72.89 {\scriptstyle \pm 0.41}}$ 
& $\mathbf{72.66 {\scriptstyle \pm 0.38}}$ 
& $\underline{94.49 {\scriptstyle \pm 0.26}}$ 
& $\mathbf{85.07 {\scriptstyle \pm 0.29}}$ 
& $\mathbf{84.42 {\scriptstyle \pm 0.29}}$ 
& $\mathbf{81.91}$ \\

\midrule

\multirow{5}{*}{Dir(0.5)}
& FedIT~\cite{FedIT} 
& $\underline{85.55 {\scriptstyle \pm 0.27}}$ 
& $\underline{85.34 {\scriptstyle \pm 0.33}}$ 
& $94.96 {\scriptstyle \pm 0.19}$ 
& $\underline{86.28 {\scriptstyle \pm 0.31}}$ 
& $89.95 {\scriptstyle \pm 0.29}$ 
& $88.42$ \\

& FFA-LoRA~\cite{FFA-LoRA} 
& $82.48 {\scriptstyle \pm 0.45}$ 
& $83.12 {\scriptstyle \pm 0.29}$ 
& $95.18 {\scriptstyle \pm 0.22}$ 
& $83.42 {\scriptstyle \pm 0.27}$ 
& $89.07 {\scriptstyle \pm 0.31}$ 
& $86.65$ \\

& FedEx-LoRA~\cite{FedEx-LoRA} 
& $84.66 {\scriptstyle \pm 0.37}$ 
& $84.64 {\scriptstyle \pm 0.30}$ 
& $\underline{95.64 {\scriptstyle \pm 0.15}}$ 
& $\underline{86.28 {\scriptstyle \pm 0.24}}$ 
& $\underline{90.94 {\scriptstyle \pm 0.29}}$ 
& $\underline{88.43}$ \\

& FedSA-LoRA~\cite{FedSA-LoRA} 
& $84.60 {\scriptstyle \pm 0.35}$ 
& $85.31 {\scriptstyle \pm 0.33}$ 
& $93.37 {\scriptstyle \pm 0.31}$ 
& $85.80 {\scriptstyle \pm 0.31}$ 
& $90.14 {\scriptstyle \pm 0.27}$ 
& $87.84$ \\

\rowcolor{oursbg}
& \textbf{GLoRA (Ours)}
& $\mathbf{85.71 {\scriptstyle \pm 0.31}}$ 
& $\mathbf{85.73 {\scriptstyle \pm 0.34}}$ 
& $\mathbf{96.33 {\scriptstyle \pm 0.13}}$ 
& $\mathbf{87.21 {\scriptstyle \pm 0.32}}$ 
& $\mathbf{91.82 {\scriptstyle \pm 0.25}}$ 
& $\mathbf{89.36}$ \\

\bottomrule
\end{tabular}
}
\end{table*}

\subsection{Overall Performance with Heterogeneous Ranks}
\label{sec:overall_hetero_rank}

\paragraph{Implementation details.}
We evaluate three representative rank distributions: normal, uniform, and heavy-tail as described in Section~\ref{sec:exp_setup_cross_device}. For each GLUE task, the training data are partitioned into 50 clients with $\mathrm{Dir}(0.5)$ label skew, and 50\% of clients participate in each round. Clients are assigned different LoRA ranks according to their resource types, while other configurations follow Section~\ref{sec:overall_homo_rank}.

\paragraph{Main results.}
As shown in Table~\ref{tab:rank_distribution}, GLoRA achieves the best average performance under all three rank distributions and ranks first on 14 out of 15 task-distribution pairs. This consistent advantage across normal, uniform, and heavy-tail rank profiles suggests that GLoRA is robust to heterogeneous adapter capacities. By aggregating updates in a shared gauge-aware subspace and redistributing them according to client capacity and update history, GLoRA better preserves intrinsic update information while remaining compatible with resource-constrained clients.

\begin{table*}[t]
\centering
\caption{Performance comparison on GLUE tasks with heterogeneous LoRA ranks under $\mathrm{Dir}(0.5)$ label skew. The best and second-best results are highlighted in \textbf{bold} and \underline{underlined}, respectively.}
\label{tab:rank_distribution}
\renewcommand{\arraystretch}{1.15}
\setlength{\tabcolsep}{5.5pt}
\definecolor{oursbg}{HTML}{F4EFFF}
\definecolor{headbg}{HTML}{F7F7F7}
\resizebox{\linewidth}{!}{
\begin{tabular}{llcccccc}
\toprule
\rowcolor{headbg}
\textbf{Distribution} & \textbf{Method} 
& \textbf{MNLI-m} & \textbf{MNLI-mm} & \textbf{SST-2} 
& \textbf{QQP} & \textbf{QNLI} & \textbf{Average} \\
\midrule

\multirow{3}{*}{Normal}
& HetLoRA~\cite{HetLoRA}
& $\underline{83.92 {\scriptstyle \pm 0.34}}$
& $\underline{84.22 {\scriptstyle \pm 0.31}}$
& $\mathbf{95.07 {\scriptstyle \pm 0.18}}$
& $\underline{83.52 {\scriptstyle \pm 0.36}}$
& $\underline{91.07 {\scriptstyle \pm 0.24}}$
& $\underline{87.56}$ \\

& FlexLoRA~\cite{FlexLoRA}
& $81.63 {\scriptstyle \pm 0.43}$
& $82.75 {\scriptstyle \pm 0.39}$
& $93.32 {\scriptstyle \pm 0.27}$
& $80.71 {\scriptstyle \pm 0.48}$
& $90.54 {\scriptstyle \pm 0.29}$
& $85.79$ \\

\rowcolor{oursbg}
& \textbf{GLoRA (Ours)}
& $\mathbf{85.57 {\scriptstyle \pm 0.28}}$
& $\mathbf{86.25 {\scriptstyle \pm 0.33}}$
& $\underline{94.04 {\scriptstyle \pm 0.23}}$
& $\mathbf{84.45 {\scriptstyle \pm 0.34}}$
& $\mathbf{91.40 {\scriptstyle \pm 0.22}}$
& $\mathbf{88.34}$ \\

\midrule

\multirow{3}{*}{Uniform}
& HetLoRA~\cite{HetLoRA}
& $\underline{82.74 {\scriptstyle \pm 0.38}}$
& $\underline{82.96 {\scriptstyle \pm 0.36}}$
& $\underline{94.61 {\scriptstyle \pm 0.20}}$
& $\underline{82.99 {\scriptstyle \pm 0.41}}$
& $\underline{89.57 {\scriptstyle \pm 0.31}}$
& $\underline{86.57}$ \\

& FlexLoRA~\cite{FlexLoRA}
& $81.89 {\scriptstyle \pm 0.46}$
& $82.74 {\scriptstyle \pm 0.42}$
& $93.67 {\scriptstyle \pm 0.18}$
& $78.57 {\scriptstyle \pm 0.57}$
& $89.54 {\scriptstyle \pm 0.34}$
& $85.28$ \\

\rowcolor{oursbg}
& \textbf{GLoRA (Ours)}
& $\mathbf{83.02 {\scriptstyle \pm 0.35}}$
& $\mathbf{83.66 {\scriptstyle \pm 0.32}}$
& $\mathbf{95.07 {\scriptstyle \pm 0.17}}$
& $\mathbf{83.39 {\scriptstyle \pm 0.39}}$
& $\mathbf{90.12 {\scriptstyle \pm 0.28}}$
& $\mathbf{87.05}$ \\

\midrule

\multirow{3}{*}{Heavy-tail}
& HetLoRA~\cite{HetLoRA}
& $\underline{85.37 {\scriptstyle \pm 0.30}}$
& $\underline{85.29 {\scriptstyle \pm 0.34}}$
& $\underline{94.18 {\scriptstyle \pm 0.25}}$
& $\underline{83.96 {\scriptstyle \pm 0.37}}$
& $\underline{90.81 {\scriptstyle \pm 0.26}}$
& $\underline{87.92}$ \\

& FlexLoRA~\cite{FlexLoRA}
& $82.71 {\scriptstyle \pm 0.44}$
& $82.72 {\scriptstyle \pm 0.41}$
& $90.09 {\scriptstyle \pm 1.06}$
& $83.17 {\scriptstyle \pm 0.43}$
& $88.54 {\scriptstyle \pm 0.36}$
& $85.45$ \\

\rowcolor{oursbg}
& \textbf{GLoRA (Ours)}
& $\mathbf{85.59 {\scriptstyle \pm 0.27}}$
& $\mathbf{85.84 {\scriptstyle \pm 0.29}}$
& $\mathbf{94.50 {\scriptstyle \pm 0.21}}$
& $\mathbf{84.70 {\scriptstyle \pm 0.33}}$
& $\mathbf{90.89 {\scriptstyle \pm 0.25}}$
& $\mathbf{88.30}$ \\

\bottomrule
\end{tabular}
}
\end{table*}

% \subsection{Generalization with Challenging Task Heterogeneity}
% \textbf{Implementation Details.} 我们这一节测试算法在任务异构叠加资源异构叠加稀疏参与场景下的性能与泛化能力。Following FlexLoRA, 我们采用Rouge-L作为生成评估指标。50个训练客户端每轮仅有10\%参与训练，每个客户端跑1个local epoch，共训练30个FL round. Local batchsize=4，Adam 优化器。lr在{1e-4, 2e-4, 5e-4, 1e-3}中search最优。Rank 异构的策略与3.3保持一致。

\subsection{Sparse Cross-Task FL with Heterogeneous Client Resources}
\label{sec:task_resource_sparse}

\vspace{-1mm}
\paragraph{Implementation details.}
We further evaluate GLoRA on SuperNI in a challenging cross-task FL setting where task heterogeneity, heterogeneous client resources, and sparse participation coexist. Following FlexLoRA~\cite{FlexLoRA}, we report ROUGE-L~\cite{ROUGE} on both in-domain training clients and unseen evaluation clients. We construct 50 training clients and 20 unseen evaluation clients from heterogeneous SuperNI tasks, with only 10\% client participation per round. Each selected client trains for one local epoch over 30 FL rounds. We use a local batch size of 4 and Adam optimizer, with the learning rate selected from $\{10^{-4}, 2\times10^{-4}, 5\times10^{-4}, 10^{-3}\}$. Rank distributions follow Section~\ref{sec:overall_hetero_rank}.

\vspace{-1mm}
\paragraph{Main results.}
Table~\ref{tab:indomain_unseen_generalization} shows that GLoRA achieves the best performance across all rank distributions on both in-domain and unseen clients. This demonstrates that GLoRA remains effective when task heterogeneity, resource heterogeneity, and sparse participation are jointly present. Beyond fitting the training clients, GLoRA also improves unseen-task generalization, suggesting that its gauge-aware subspace aggregation better preserves transferable update information than rank-specific factor aggregation or truncation. Fig.~\ref{fig:SuperNI_task} further confirms this trend at the task-category level, where GLoRA performs strongly across most seen and unseen categories under different rank distributions.

\begin{table*}[t]
\centering
\caption{ROUGE-L comparison on SuperNI under sparse participation and heterogeneous LoRA ranks. Results are reported on both in-domain training clients and unseen evaluation clients.}
\label{tab:indomain_unseen_generalization}
\renewcommand{\arraystretch}{1.15}
\setlength{\tabcolsep}{6.2pt}
\definecolor{oursbg}{HTML}{F4EFFF}
\definecolor{headbg}{HTML}{F7F7F7}
\resizebox{\linewidth}{!}{
\begin{tabular}{lcccccc}
\toprule
\rowcolor{headbg}
\multirow{2}{*}{\textbf{Method}} 
& \multicolumn{3}{c}{\textbf{In-domain Client Performance}} 
& \multicolumn{3}{c}{\textbf{Unseen Task Generalization}} \\
\cmidrule(lr){2-4} \cmidrule(lr){5-7}
\rowcolor{headbg}
& \textbf{Normal} & \textbf{Uniform} & \textbf{Heavy-tail}
& \textbf{Normal} & \textbf{Uniform} & \textbf{Heavy-tail} \\
\midrule

HetLoRA
& $\underline{50.31 {\scriptstyle \pm 0.42}}$
& $\underline{48.59 {\scriptstyle \pm 0.51}}$
& $\underline{49.49 {\scriptstyle \pm 0.48}}$
& $\underline{35.95 {\scriptstyle \pm 0.63}}$
& $\underline{33.04 {\scriptstyle \pm 0.78}}$
& $\underline{33.88 {\scriptstyle \pm 0.72}}$ \\

FlexLoRA
& $48.72 {\scriptstyle \pm 0.57}$
& $42.53 {\scriptstyle \pm 0.84}$
& $44.10 {\scriptstyle \pm 0.76}$
& $35.12 {\scriptstyle \pm 0.71}$
& $29.21 {\scriptstyle \pm 1.02}$
& $31.11 {\scriptstyle \pm 0.91}$ \\

\rowcolor{oursbg}
\textbf{GLoRA (Ours)}
& $\mathbf{52.26 {\scriptstyle \pm 0.36}}$
& $\mathbf{49.61 {\scriptstyle \pm 0.44}}$
& $\mathbf{50.42 {\scriptstyle \pm 0.41}}$
& $\mathbf{37.97 {\scriptstyle \pm 0.58}}$
& $\mathbf{34.95 {\scriptstyle \pm 0.69}}$
& $\mathbf{34.50 {\scriptstyle \pm 0.73}}$ \\

\bottomrule
\end{tabular}
}
\end{table*}

\begin{figure}
    \centering
    \includegraphics[width=0.8\linewidth]{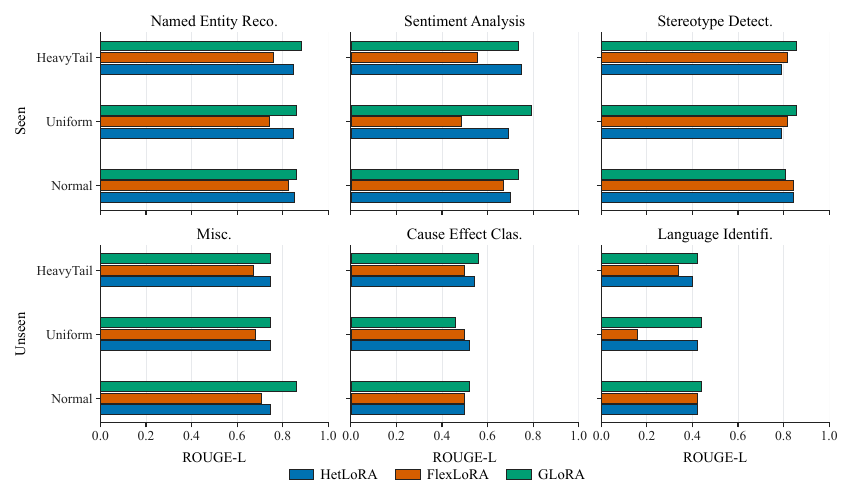}
\vspace{-1mm}
\caption{Per-category ROUGE-L performance on SuperNI across seen and unseen task categories.}
\label{fig:SuperNI_task}
\vspace{-3mm}
\end{figure}

\vspace{-1mm}
\subsection{In-depth Analysis}
\label{sec:indepth_analysis}

\vspace{-1mm}
\paragraph{Q1: How does gauge ambiguity affect federated LoRA optimization?}
As shown in Fig.~\ref{fig:conv_sst2}, methods that aggregate raw LoRA factors exhibit more unstable convergence under strong data heterogeneity, e.g., $\mathrm{Dir}(0.1)$. This is because gauge-equivalent LoRA updates can have different factor coordinates, making factor-level aggregation ill-defined. GLoRA mitigates this issue by aggregating intrinsic low-rank updates in a gauge-aware subspace, leading to smoother optimization. We note that GLoRA does not eliminate the inherent difficulty of data or rank heterogeneity; rather, it removes the additional instability introduced by gauge ambiguity.

\vspace{-1mm}
\paragraph{Q2: How does the rank budget $R$ affect performance?}
$R$ determines the size of the server consensus subspace. In experiments, we control it by a ratio $\rho$, i.e., $R=\rho\sum_{i\in\mathcal C_t} r_i$, where $r_i$ is the rank of participating client $i$. When $R$ covers the total participating rank, GLoRA recovers exact update aggregation; otherwise, it discards tail update directions. Fig.~\ref{fig:rank_budget} shows that GLoRA degrades gracefully as $R$ decreases on both GLUE and SuperNI, suggesting that leading consensus directions retain most useful update information and provide a practical efficiency--performance trade-off.

\vspace{-1mm}
\paragraph{Q3: How important is client-aware readout (CAR)?}
Fig.~\ref{fig:core_lr_landscape} shows that the core ratio matters more under severe data heterogeneity, where client updates contain stronger client-specific components. Nevertheless, GLoRA remains robust across a broad range of ratios. Under milder heterogeneity, performance becomes less sensitive, indicating that CAR does not require delicate tuning.

\vspace{-1mm}
\paragraph{Q4: Does GLoRA scale to larger backbones?}
Table~\ref{tab:large_backbone} reports additional GLUE results on Qwen2-7B and Gemma-2-9B with the same setting in Sec.~\ref{sec:overall_homo_rank}. GLoRA consistently maintains its advantage, suggesting that gauge-aware aggregation is not limited to small or medium backbones but remains effective for larger models.

\vspace{-1mm}
\paragraph{Q5: Are the core designs necessary?}
Table~\ref{tab:ablation_core_designs} ablates projected aggregation (PA) and client-aware readout (CAR). Removing either component degrades performance: projected aggregation provides a well-defined aggregation space, while client-aware readout adapts the shared update representation to heterogeneous client capacities. Their combination yields the strongest overall performance. 

\vspace{-1mm}
\paragraph{Q6: What is the efficiency of GLoRA?}
Table~\ref{tab:efficiency_comparison} shows that GLoRA achieves a favorable efficiency--effectiveness trade-off. It supports heterogeneous ranks without gauge-dependent factor aggregation and avoids the dense-update operations required by FlexLoRA. Per LoRA layer, its cost is $\Theta\!\left(r_\Sigma^2(d_{\mathrm{out}}+d_{\mathrm{in}})+r_\Sigma^3\right)$, where $r_\Sigma=\sum_{i\in\mathcal C_t}r_i$, making it much cheaper than dense-update aggregation. The ``Rounds'' column reports the rounds needed to reach $0.92$ accuracy on SST-2 under uniform ranks, where GLoRA with $\beta=0.5$ converges fastest. Details are in Appendix.

\begin{figure*}[t]
\centering
\captionsetup{font=small,labelfont=bf,skip=2pt}
\definecolor{oursbg}{HTML}{F4EFFF}
\definecolor{headbg}{HTML}{F7F7F7}

% ===================== Left column =====================
\begin{minipage}[t]{0.48\linewidth}
    \centering

    \includegraphics[width=\linewidth]{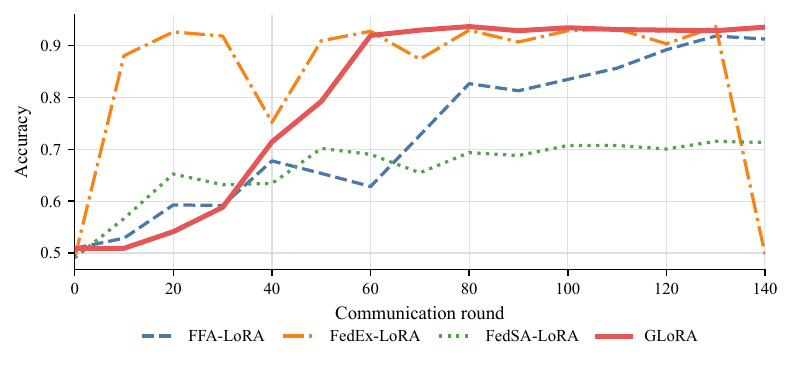}
    \captionof{figure}{Convergence on SST-2 under $\mathrm{Dir}(0.1)$.}
    \label{fig:conv_sst2}

    \vspace{0.55em}

    \includegraphics[width=0.92\linewidth]{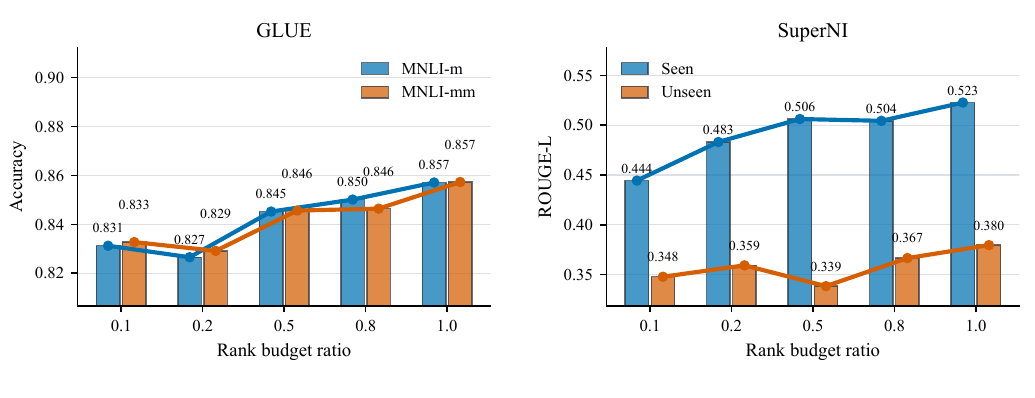}
    \captionof{figure}{Sensitivity to rank budget $R$.}
    \label{fig:rank_budget}
\end{minipage}
\hfill
% ===================== Right column =====================
\begin{minipage}[t]{0.50\linewidth}
    \centering

    \includegraphics[width=\linewidth]{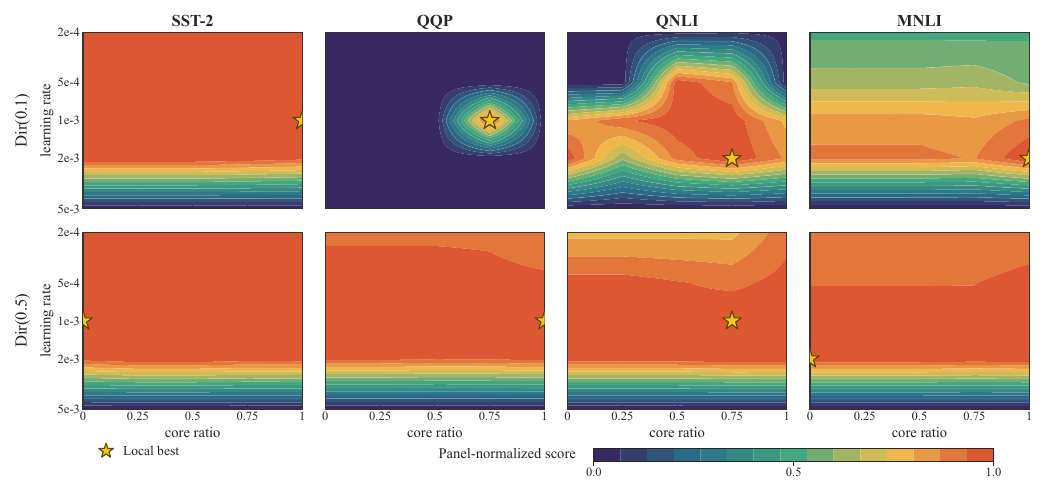}
    % \captionof{figure}{Core-ratio and learning-rate landscape.}
    \captionof{figure}{Hyperparameter landscape over core ratio $\gamma$.}
    \label{fig:core_lr_landscape}

    \vspace{0.1em}

    \captionof{table}{Ablation study of GLoRA. CAR is not applicable without consensus subspaces and is marked as N/A.}
    \label{tab:ablation_core_designs}

    \renewcommand{\arraystretch}{1.05}
    \setlength{\tabcolsep}{3.2pt}
    \resizebox{\linewidth}{!}{
    \begin{tabular}{cccccccc}
    \toprule
    \rowcolor{headbg}
    \multicolumn{2}{c}{\textbf{Design}} 
    & \multicolumn{3}{c}{\textbf{Homo Dir(0.1)}} 
    & \multicolumn{3}{c}{\textbf{Hetero Dir(0.5) Uniform}} \\
    \cmidrule(lr){1-2}
    \cmidrule(lr){3-5}
    \cmidrule(lr){6-8}
    \rowcolor{headbg}
    \textbf{PA} & \textbf{CAR} 
    & \textbf{MNLI-m} & \textbf{SST2} & \textbf{QNLI} 
    & \textbf{MNLI-m} & \textbf{SST2} & \textbf{QNLI} \\
    \midrule

    \xmark & \xmark 
    & $71.63$ & $93.98$ & $83.18$ 
    & $81.47$ & $94.06$ & $88.31$ \\

    \xmark & \cmark 
    & N/A & N/A & N/A 
    & N/A & N/A & N/A \\

    \cmark & \xmark 
    & $72.08$ & $94.22$ & $83.72$ 
    & $82.84$ & $94.47$ & $89.36$ \\

    \rowcolor{oursbg}
    \cmark & \cmark 
    & $\mathbf{72.84}$ & $\mathbf{94.53}$ & $\mathbf{84.37}$ 
    & $\mathbf{83.15}$ & $\mathbf{94.96}$ & $\mathbf{90.03}$ \\

    \bottomrule
    \end{tabular}
    }
\end{minipage}

\end{figure*}
\begin{table*}[t]
\centering
\captionsetup{font=small,labelfont=bf,skip=2pt}
\definecolor{oursbg}{HTML}{F4EFFF}
\definecolor{headbg}{HTML}{F7F7F7}

\begin{minipage}[t]{0.32\linewidth}
\centering
\caption{Larger-backbone performance under $\mathrm{Dir}(0.1)$.}
\label{tab:large_backbone}
\renewcommand{\arraystretch}{1.10}
\setlength{\tabcolsep}{4.6pt}
\resizebox{\linewidth}{!}{
\begin{tabular}{llcc}
\toprule
\rowcolor{headbg}
\textbf{Backbone} & \textbf{Method} & \textbf{QNLI} & \textbf{QQP} \\
\midrule

\multirow{4}{*}{Qwen2-7B}
& FedIT      & $86.62$ & $86.21$ \\
& FFA-LoRA   & $84.93$ & $84.67$ \\
& FedEx-LoRA & $\underline{87.38}$ & $\underline{86.94}$ \\
\rowcolor{oursbg}
& \textbf{GLoRA} & $\mathbf{88.24}$ & $\mathbf{87.73}$ \\

\midrule

\multirow{4}{*}{Gemma-2-9B}
& FedIT      & $87.25$ & $86.74$ \\
& FFA-LoRA   & $85.61$ & $85.18$ \\
& FedEx-LoRA & $\underline{88.02}$ & $\underline{87.42}$ \\
\rowcolor{oursbg}
& \textbf{GLoRA} & $\mathbf{88.83}$ & $\mathbf{88.28}$ \\

\bottomrule
\end{tabular}
}
\end{minipage}
\hfill
\begin{minipage}[t]{0.60\linewidth}
\centering
\caption{Efficiency comparison of federated LoRA aggregation methods. Homo: SST-2; Hetero: SuperNI with normal ranks.}
\label{tab:efficiency_comparison}
\renewcommand{\arraystretch}{1.12}
\setlength{\tabcolsep}{4.2pt}
\resizebox{\linewidth}{!}{
% \begin{tabular}{lccccc}
% \toprule
% \rowcolor{headbg}
% \multirow{2}{*}{\textbf{Method}}
% & \multicolumn{2}{c}{\textbf{Property}}
% & \multicolumn{3}{c}{\textbf{Efficiency}} \\
% \cmidrule(lr){2-3} \cmidrule(lr){4-6}
% \rowcolor{headbg}
% & \textbf{Complexity}
% & \textbf{Hetero}
% & \textbf{Homo}
% & \textbf{Hetero}
% & \textbf{Rounds} \\
% \midrule

% FedIT
% & $\Theta(nr(d+k))$
% & \xmark
% & $1.3$s
% & N/A
% & N/A \\

% HetLoRA
% & $\Theta(dkS)$
% & \cmark
% & $4.4$s
% & $2.9$s
% & 142 \\

% FlexLoRA
% & $\Theta(dk(S+\min(d,k)))$
% & \cmark
% & $70.7$s
% & $34.0$s
% & 275 \\

% \rowcolor{oursbg}
% \textbf{GLoRA} $(\beta=0.1)$
% & $\Theta(dS^2 + S^3 + SR(d+k))$
% & \cmark
% & $\mathbf{2.2}$s
% & $\mathbf{1.7}$s
% & 154 \\

% \rowcolor{oursbg}
% \textbf{GLoRA} $(\beta=0.5)$
% & $\Theta(dS^2 + S^3 + SR(d+k))$
% & \cmark
% & $3.7$s
% & $2.7$s
% & \textbf{104} \\

% \rowcolor{oursbg}
% \textbf{GLoRA} $(\beta=1)$
% & $\Theta(dS^2 + S^3 + SR(d+k))$
% & \cmark
% & $8.3$s
% & $12.3$s
% & 109 \\

% \bottomrule
% \end{tabular}
\begin{tabular}{lccccc}
\toprule
\rowcolor{headbg}
\multirow{2}{*}{\textbf{Method}}
& \multicolumn{2}{c}{\textbf{Property}}
& \multicolumn{3}{c}{\textbf{Efficiency}} \\
\cmidrule(lr){2-3} \cmidrule(lr){4-6}
\rowcolor{headbg}
& \textbf{Per-layer complexity}
& \textbf{Hetero}
& \textbf{Homo}
& \textbf{Hetero}
& \textbf{Rounds} \\
\midrule

FedIT
& $\Theta\!\left(nr(d_{\mathrm{out}}+d_{\mathrm{in}})\right)$
& \xmark
& $1.3$s
& N/A
& N/A \\

HetLoRA
& $\Theta\!\left(r_\Sigma d_{\mathrm{out}}d_{\mathrm{in}}\right)$
& \cmark
& $4.4$s
& $2.9$s
& 142 \\

FlexLoRA
& $\Theta\!\left(d_{\mathrm{out}}d_{\mathrm{in}}(r_\Sigma + m)\right)$
& \cmark
& $70.7$s
& $34.0$s
& 275 \\

\rowcolor{oursbg}
\textbf{GLoRA} $(\beta=0.1)$
& $\Theta\!\left(r_\Sigma^2(d_{\mathrm{out}}+d_{\mathrm{in}})+r_\Sigma^3\right)$
& \cmark
& $\mathbf{2.2}$s
& $\mathbf{1.7}$s
& 154 \\

\rowcolor{oursbg}
\textbf{GLoRA} $(\beta=0.5)$
& $\Theta\!\left(r_\Sigma^2(d_{\mathrm{out}}+d_{\mathrm{in}})+r_\Sigma^3\right)$
& \cmark
& $3.7$s
& $2.7$s
& \textbf{104} \\

\rowcolor{oursbg}
\textbf{GLoRA} $(\beta=1)$
& $\Theta\!\left(r_\Sigma^2(d_{\mathrm{out}}+d_{\mathrm{in}})+r_\Sigma^3\right)$
& \cmark
& $8.3$s
& $12.3$s
& 109 \\

\bottomrule
\end{tabular}
}
\end{minipage}
\end{table*}

\vspace{-1mm}
\section{Conclusion}

We presented \textbf{GLoRA}, a gauge-aware low-rank server representation for federated LoRA aggregation. By separating update geometry from factor coordinates, GLoRA avoids gauge-dependent factor aggregation, performs semantic aggregation in a shared consensus subspace, and supports heterogeneous clients through rank-compatible readout without dense-update materialization. Experiments on GLUE and SuperNI show that GLoRA consistently improves over federated LoRA baselines under data, resource, and task heterogeneity, while maintaining a favorable efficiency--performance trade-off. These results suggest that effective federated LoRA depends not only on averaging low-rank factors, but also on defining a meaningful server-side aggregation object. One limitation is that the current consensus subspace is estimated mainly from client column-space geometry; incorporating update-energy or task-aware weighting into subspace construction is an interesting direction for future work.

%% file: sections/appendix.tex
\newpage
\appendix
\section{Notation and Proofs}
\subsection{Proof of Proposition~\ref{prop:exact_sufficient_rank}}

\begin{proof}
\label{app:exactness_proof}
Recall that each participating client update can be written in the
subspace--coordinate form
\[
\Delta W_i^t = U_i^t \widehat A_i^t,
\qquad
(U_i^t)^\top U_i^t = I_{r_i}.
\]
Let
\[
\mathcal S^t
=
\mathrm{span}\big([U_i^t]_{i\in\mathcal C_t}\big)
\]
denote the union span of all participating client update subspaces, and let
\[
r_\cup^t=\dim(\mathcal S^t)
=
\mathrm{rank}\big([U_i^t]_{i\in\mathcal C_t}\big).
\]
The matrix used to estimate the consensus subspace is
\[
K^t
=
\sum_{i\in\mathcal C_t}p_i^t U_i^t(U_i^t)^\top .
\]
Equivalently, define the thin matrix
\[
M^t=\big[\sqrt{p_i^t}U_i^t\big]_{i\in\mathcal C_t}.
\]
Then
\[
K^t = M^t(M^t)^\top .
\]
Since \(p_i^t>0\) for participating clients, the column span of \(M^t\) is exactly
\(\mathcal S^t\). Therefore,
\[
\mathrm{range}(K^t)
=
\mathrm{range}(M^t)
=
\mathcal S^t,
\]
and hence
\[
\mathrm{rank}(K^t)=r_\cup^t.
\]

When the server rank satisfies \(R\ge r_\cup^t\), the reference basis
\[
U_{\mathrm{ref}}^t=\mathrm{TopEig}_R(K^t)
\]
contains an orthonormal basis of the entire nonzero eigenspace of \(K^t\).
Since the nonzero eigenspace of \(K^t\) is precisely \(\mathcal S^t\), we have
\[
\mathcal S^t \subseteq \mathrm{span}(U_{\mathrm{ref}}^t).
\]
Thus, for every client \(i\in\mathcal C_t\),
\[
U_{\mathrm{ref}}^t(U_{\mathrm{ref}}^t)^\top U_i^t = U_i^t,
\]
because every column of \(U_i^t\) lies in \(\mathcal S^t\).

By the definition of the aggregated coordinates,
\[
Z_g^t
=
\sum_{i\in\mathcal C_t}p_i^t Z_i^t
=
\sum_{i\in\mathcal C_t}
p_i^t (U_{\mathrm{ref}}^t)^\top U_i^t \widehat A_i^t .
\]
Multiplying both sides by \(U_{\mathrm{ref}}^t\), we obtain
\[
U_{\mathrm{ref}}^t Z_g^t
=
\sum_{i\in\mathcal C_t}
p_i^t
U_{\mathrm{ref}}^t(U_{\mathrm{ref}}^t)^\top U_i^t \widehat A_i^t .
\]
Using
\[
U_{\mathrm{ref}}^t(U_{\mathrm{ref}}^t)^\top U_i^t = U_i^t,
\]
this becomes
\[
U_{\mathrm{ref}}^t Z_g^t
=
\sum_{i\in\mathcal C_t}
p_i^t U_i^t \widehat A_i^t
=
\sum_{i\in\mathcal C_t}
p_i^t \Delta W_i^t .
\]
Therefore, GLoRA exactly recovers the dense update average whenever
\(R\ge r_\cup^t\).

Finally, since the rank of a concatenation is at most the sum of the ranks of its
blocks,
\[
r_\cup^t
=
\mathrm{rank}\big([U_i^t]_{i\in\mathcal C_t}\big)
\le
\sum_{i\in\mathcal C_t} r_i,
\]
the stronger condition
\[
R\ge \sum_{i\in\mathcal C_t}r_i
\]
is sufficient for exact aggregation.
\end{proof}

\subsection{Proof of Proposition~\ref{prop:gauge_invariance}}

\begin{proof}
Consider arbitrary client-side gauge transformations
\[
B_i^t \mapsto \widetilde B_i^t = B_i^t Q_i,
\qquad
A_i^t \mapsto \widetilde A_i^t = Q_i^{-1}A_i^t,
\]
where \(Q_i\in \mathrm{GL}(r_i)\) is invertible. The induced LoRA update is
unchanged:
\[
\widetilde B_i^t \widetilde A_i^t
=
B_i^t Q_i Q_i^{-1} A_i^t
=
B_i^t A_i^t
=
\Delta W_i^t .
\]
Thus, the intrinsic client update matrix remains the same.

Next, we show that the consensus subspace is also unchanged. Let
\[
B_i^t=U_i^tT_i^t
\]
be the reduced QR decomposition used by GLoRA. Since \(Q_i\) is invertible,
\[
\mathrm{span}(\widetilde B_i^t)
=
\mathrm{span}(B_i^tQ_i)
=
\mathrm{span}(B_i^t)
=
\mathrm{span}(U_i^t).
\]
Therefore, the orthonormal basis obtained from \(\widetilde B_i^t\) may differ
from \(U_i^t\), but it spans the same column space. Hence there exists an
orthogonal matrix \(S_i\in\mathbb R^{r_i\times r_i}\) such that
\[
\widetilde U_i^t = U_i^t S_i.
\]
Consequently, the corresponding orthogonal projector is invariant:
\[
\widetilde P_i^t
=
\widetilde U_i^t(\widetilde U_i^t)^\top
=
U_i^t S_i S_i^\top (U_i^t)^\top
=
U_i^t(U_i^t)^\top
=
P_i^t .
\]
It follows that the weighted projector aggregation matrix is unchanged:
\[
\widetilde K^t
=
\sum_{i\in\mathcal C_t}p_i^t \widetilde P_i^t
=
\sum_{i\in\mathcal C_t}p_i^t P_i^t
=
K^t .
\]
Therefore, the reference subspace computed from \(K^t\) is unchanged. In
particular, \(U_{\mathrm{ref}}^t\) spans the same consensus subspace before and
after the gauge transformation.

The GLoRA server update can be written as
\[
\Delta W_g^t
=
U_{\mathrm{ref}}^t(U_{\mathrm{ref}}^t)^\top
\sum_{i\in\mathcal C_t}p_i^t\Delta W_i^t .
\]
Both terms in this expression are invariant under the gauge transformations:
the projector
\[
U_{\mathrm{ref}}^t(U_{\mathrm{ref}}^t)^\top
\]
is unchanged because the consensus subspace is unchanged, and each intrinsic
client update \(\Delta W_i^t\) is unchanged as shown above. Therefore,
\[
\widetilde{\Delta W}_g^t
=
\Delta W_g^t .
\]
Hence, the GLoRA server update is invariant to arbitrary client-side gauge
reparameterizations.
\end{proof}
\begin{table}[t]
\centering
\caption{Summary of notation.}
\label{tab:notation}
\resizebox{0.85\linewidth}{!}{
\begin{tabular}{lll}
\toprule
Symbol & Meaning & Shape \\
\midrule
$t$ & Communication round & Scalar \\
$i$ & Client index & Scalar \\
$\mathcal C_t$ & Set of active clients at round $t$ & Set \\
$p_i^t$ & Aggregation weight of client $i$ at round $t$ & Scalar \\
$d_{\mathrm{in}}$ & Input dimension of the adapted linear layer & Scalar \\
$d_{\mathrm{out}}$ & Output dimension of the adapted linear layer & Scalar \\
$r$ & LoRA rank in the homogeneous setting & Scalar \\
$r_i$ & LoRA rank of client $i$ & Scalar \\
$R$ & Server-side rank budget & Scalar \\
\midrule
$\Delta W$ & LoRA adaptation matrix & $\mathbb{R}^{d_{\mathrm{out}}\times d_{\mathrm{in}}}$ \\
$B$ & LoRA left factor & $\mathbb{R}^{d_{\mathrm{out}}\times r}$ \\
$A$ & LoRA right factor & $\mathbb{R}^{r\times d_{\mathrm{in}}}$ \\
$Q$ & Invertible gauge transformation matrix & $\mathbb{R}^{r\times r}$ \\
$\mathrm{GL}(r)$ & General linear group of invertible $r\times r$ matrices & Set \\
$[(B,A)]$ & Gauge-equivalence class of LoRA factorizations & Set \\
\midrule
$\Delta W_i^t$ & Local LoRA update of client $i$ at round $t$ & $\mathbb{R}^{d_{\mathrm{out}}\times d_{\mathrm{in}}}$ \\
$B_i^t$ & Client $i$'s LoRA left factor at round $t$ & $\mathbb{R}^{d_{\mathrm{out}}\times r_i}$ \\
$A_i^t$ & Client $i$'s LoRA right factor at round $t$ & $\mathbb{R}^{r_i\times d_{\mathrm{in}}}$ \\
$Q_i$ & Client-specific gauge transformation & $\mathbb{R}^{r_i\times r_i}$ \\
$\Omega^t$ & Server state at round $t$ & Tuple \\
$\Psi$ & Server update rule & Function \\
\midrule
$U_i^t$ & Orthonormal basis of client $i$'s update subspace & $\mathbb{R}^{d_{\mathrm{out}}\times r_i}$ \\
$T_i^t$ & Upper-triangular factor from reduced QR decomposition of $B_i^t$ & $\mathbb{R}^{r_i\times r_i}$ \\
$\widehat A_i^t$ & Gauge-fixed coordinate matrix of client $i$ & $\mathbb{R}^{r_i\times d_{\mathrm{in}}}$ \\
$I_{r_i}$ & Identity matrix of size $r_i$ & $\mathbb{R}^{r_i\times r_i}$ \\
$P_i^t$ & Orthogonal projector onto client $i$'s update subspace & $\mathbb{R}^{d_{\mathrm{out}}\times d_{\mathrm{out}}}$ \\
$K^t$ & Weighted covariance / projector aggregation matrix for subspace estimation & $\mathbb{R}^{d_{\mathrm{out}}\times d_{\mathrm{out}}}$ \\
$M^t$ & Thin matrix used to obtain the eigenspace of $K^t$ efficiently & $\mathbb{R}^{d_{\mathrm{out}}\times \sum_{i\in\mathcal C_t} r_i}$ \\
$U_{\mathrm{ref}}^t$ & Server reference basis / consensus subspace & $\mathbb{R}^{d_{\mathrm{out}}\times R}$ \\
$I_R$ & Identity matrix of size $R$ & $\mathbb{R}^{R\times R}$ \\
$r_{\cup}^t$ & Dimension of the union span of participating client subspaces & Scalar \\
\midrule
$Z_i^t$ & Client $i$'s projected coordinates in the reference basis & $\mathbb{R}^{R\times d_{\mathrm{in}}}$ \\
$Z_g^t$ & Aggregated server coordinates in the reference basis & $\mathbb{R}^{R\times d_{\mathrm{in}}}$ \\
$\Delta W_g^t$ & Gauge-aware server update induced by $(U_{\mathrm{ref}}^t,Z_g^t)$ & $\mathbb{R}^{d_{\mathrm{out}}\times d_{\mathrm{in}}}$ \\
\midrule
$O^t$ & Left singular vectors of $Z_g^t$ in the coordinate space & $\mathbb{R}^{R\times \ell}$ \\
$\Sigma^t$ & Singular value matrix of $Z_g^t$ & $\mathbb{R}^{\ell\times \ell}$ \\
$V^t$ & Right singular vectors of $Z_g^t$ & $\mathbb{R}^{d_{\mathrm{in}}\times \ell}$ \\
$U_s^t$ & Server spectral basis after mapping back to the original output space & $\mathbb{R}^{d_{\mathrm{out}}\times \ell}$ \\
$\ell$ & Effective spectral rank, $\ell=\min(R,d_{\mathrm{in}})$ & Scalar \\
$\sigma_j^t$ & The $j$-th singular value of the server update & Scalar \\
$u_j^t$ & The $j$-th left spectral direction of the server update & $\mathbb{R}^{d_{\mathrm{out}}}$ \\
$v_j^t$ & The $j$-th right spectral direction of the server update & $\mathbb{R}^{d_{\mathrm{in}}}$ \\
\midrule
$\gamma$ & Core ratio for client-aware readout & Scalar in $[0,1]$ \\
$\mathcal G_i^t$ & Core spectral component set for client $i$ & Set \\
$H_i^{t-}$ & Latest historical update subspace basis of client $i$ & $\mathbb{R}^{d_{\mathrm{out}}\times h_i}$ \\
$a_{ij}^t$ & Alignment score between client $i$'s history and spectral direction $j$ & Scalar \\
$\mathcal I_i^t$ & Selected spectral component set for client $i$ & Set, $|\mathcal I_i^t|=r_i$ \\
$B_{i,\mathrm{init}}^t$ & Redistributed LoRA left factor for client $i$ & $\mathbb{R}^{d_{\mathrm{out}}\times r_i}$ \\
$A_{i,\mathrm{init}}^t$ & Redistributed LoRA right factor for client $i$ & $\mathbb{R}^{r_i\times d_{\mathrm{in}}}$ \\
\bottomrule
\end{tabular}
}
\end{table}

\begin{algorithm}[t]
\caption{GLoRA: Gauge-aware LoRA Aggregation}
\label{alg:glora}
\begin{algorithmic}[1]
\Require Initial model $W^0$; total communication rounds $T$; server rank budget $R$;
client ranks $\{r_i\}$; core ratio $\gamma$.
\State Initialize client adapter initializations and previous-round subspace bases $\{H_i^{0-}\}$.

\For{$t=0,1,\ldots,T-1$}
    \State Sample active client set $\mathcal C_t$.
    \State Send the current global model and the latest available adapter initialization
    to each client $i\in\mathcal C_t$.

    \ForAll{client $i\in\mathcal C_t$ \textbf{in parallel}}
        \State Perform local training and obtain LoRA factors
        \[
        B_i^t\in\mathbb R^{d_{\mathrm{out}}\times r_i},
        \qquad
        A_i^t\in\mathbb R^{r_i\times d_{\mathrm{in}}}.
        \]
        \State Compute reduced QR decomposition:
        \[
        B_i^t=U_i^tT_i^t,
        \qquad
        (U_i^t)^\top U_i^t=I_{r_i}.
        \]
        \State Compute gauge-fixed coordinates:
        \[
        \widehat A_i^t=T_i^tA_i^t.
        \]
        \State Send $(U_i^t,\widehat A_i^t)$ to the server.
    \EndFor

    \State $(U_{\mathrm{ref}}^t,Z_g^t)
    \leftarrow
    \textsc{ProjectedAggregate}
    \big(\{U_i^t,\widehat A_i^t,p_i^t\}_{i\in\mathcal C_t},R\big)$.

    \State $\{(B_{i,\mathrm{init}}^t,A_{i,\mathrm{init}}^t)\}_{i\in\mathcal D_t}
    \leftarrow
    \textsc{ClientAwareReadout}
    \big(U_{\mathrm{ref}}^t,Z_g^t,\{r_i\}_{i\in\mathcal D_t},
    \gamma,\{H_i^{t-}\}_{i\in\mathcal D_t}\big)$.

    \State Store current client subspaces for future rounds:
    \[
    H_i^{(t+1)-}\leftarrow U_i^t,\qquad \forall i\in\mathcal C_t.
    \]
\EndFor

\Ensure Final global model.
\end{algorithmic}
\end{algorithm}

\begin{algorithm}[t]
\caption{\textsc{ProjectedAggregate}}
\label{alg:projected_aggregate}
\begin{algorithmic}[1]
\Require Gauge-fixed client updates $\{U_i^t,\widehat A_i^t,p_i^t\}_{i\in\mathcal C_t}$;
server rank budget $R$.

\State Construct the thin weighted subspace matrix:
\[
M^t=
\big[\sqrt{p_i^t}U_i^t\big]_{i\in\mathcal C_t}.
\]

\State Compute the consensus reference basis:
\[
U_{\mathrm{ref}}^t
=
\mathrm{TopEig}_R\big(M^t(M^t)^\top\big),
\qquad
(U_{\mathrm{ref}}^t)^\top U_{\mathrm{ref}}^t=I_R.
\]

\ForAll{client $i\in\mathcal C_t$}
    \State Project the client update into the shared reference frame:
    \[
    Z_i^t
    =
    (U_{\mathrm{ref}}^t)^\top U_i^t\widehat A_i^t.
    \]
\EndFor

\State Aggregate shared coordinates:
\[
Z_g^t
=
\sum_{i\in\mathcal C_t}p_i^t Z_i^t.
\]

\State \Return $U_{\mathrm{ref}}^t$ and $Z_g^t$.
\end{algorithmic}
\end{algorithm}

\begin{algorithm}[t]
\caption{\textsc{ClientAwareReadout}}
\label{alg:client_aware_readout}
\begin{algorithmic}[1]
\Require Server state $(U_{\mathrm{ref}}^t,Z_g^t)$; target client ranks $\{r_i\}_{i\in\mathcal D_t}$;
core ratio $\gamma$; previous-round bases $\{H_i^{t-}\}_{i\in\mathcal D_t}$.

\State Compute thin SVD of the server coordinate matrix:
\[
Z_g^t=O^t\Sigma^t(V^t)^\top,
\qquad
\ell=\min(R,d_{\mathrm{in}}).
\]

\State Map the spectral basis back to the original output space:
\[
U_s^t=U_{\mathrm{ref}}^tO^t.
\]

\ForAll{client $i\in\mathcal D_t$}
    \State Define the shared core component set:
    \[
    \mathcal G_i^t
    =
    \{1,\ldots,\lfloor\gamma r_i\rfloor\}.
    \]

    \If{$H_i^{t-}$ is available}
        \State Compute alignment scores for non-core components:
        \[
        a_{ij}^t
        =
        \|(H_i^{t-})^\top u_j^t\|_2^2,
        \qquad j\notin\mathcal G_i^t.
        \]
        \State Select client-aware components:
        \[
        \mathcal I_i^t
        =
        \mathcal G_i^t
        \cup
        \mathrm{TopK}_{r_i-|\mathcal G_i^t|}
        \{a_{ij}^t:j\notin\mathcal G_i^t\}.
        \]
    \Else
        \State Select components by global spectral order:
        \[
        \mathcal I_i^t=\{1,\ldots,r_i\}.
        \]
    \EndIf

    \State Initialize client $i$ with balanced LoRA factors:
    \[
    B_{i,\mathrm{init}}^t
    =
    U_s^t[:,\mathcal I_i^t]
    \mathrm{diag}\!\left(\sqrt{\sigma_{\mathcal I_i^t}^t}\right),
    \]
    \[
    A_{i,\mathrm{init}}^t
    =
    \mathrm{diag}\!\left(\sqrt{\sigma_{\mathcal I_i^t}^t}\right)
    \left(V^t[:,\mathcal I_i^t]\right)^\top .
    \]
\EndFor

\State \Return $\{(B_{i,\mathrm{init}}^t,A_{i,\mathrm{init}}^t)\}_{i\in\mathcal D_t}$.
\end{algorithmic}
\end{algorithm}

\section{Pseudo-code of GLoRA}
We here present the pseudo-code of GLoRA. Here $\mathcal D_t$ denotes the set of clients for which the server prepares
rank-compatible adapter initializations. In practice, readout can be performed
lazily when a client is selected in the next round.
\section{Server-side Complexity}
\label{app:server_complexity}

We analyze the per-module server-side complexity of different federated LoRA
aggregation methods. Consider one LoRA module whose induced update has dimension
\[
\Delta W_i^t \in \mathbb{R}^{d_{\mathrm{out}}\times d_{\mathrm{in}}}.
\]
Let \(\mathcal C_t\) denote the participating client set at round \(t\), and
let \(r_i\) be the LoRA rank of client \(i\). Following the notation in the
main text, we define the total participating rank as
\[
r_\Sigma=\sum_{i\in\mathcal C_t}r_i .
\]
For homogeneous-rank methods, \(r_i=r\) and \(r_\Sigma=|\mathcal C_t|r\).
We also denote
\[
m=\min(d_{\mathrm{out}},d_{\mathrm{in}}).
\]
For the full model, the total cost is obtained by summing the following
per-module costs over all LoRA modules.

\paragraph{FedIT.}
FedIT directly averages homogeneous LoRA factors
\[
A_i^t\in\mathbb{R}^{r\times d_{\mathrm{in}}},
\qquad
B_i^t\in\mathbb{R}^{d_{\mathrm{out}}\times r}.
\]
For each client, the server accumulates
\(r d_{\mathrm{in}}+d_{\mathrm{out}}r\) entries. Therefore, the per-module
server cost is
\[
\mathcal{C}_{\mathrm{FedIT}}
=
\Theta\bigl(|\mathcal C_t|r(d_{\mathrm{out}}+d_{\mathrm{in}})\bigr)
=
\Theta\bigl(r_\Sigma(d_{\mathrm{out}}+d_{\mathrm{in}})\bigr).
\]
This method is efficient but only supports homogeneous ranks.

\paragraph{HetLoRA.}
HetLoRA supports heterogeneous ranks through truncation-based factor
aggregation. The factor aggregation and rank readout terms scale linearly with
the number of transmitted factor parameters. In the sparse-weighted version
considered in this paper, the server additionally computes weights based on
the dense low-rank products \(B_i^tA_i^t\). For client \(i\), forming
\[
B_i^tA_i^t:
(d_{\mathrm{out}}\times r_i)(r_i\times d_{\mathrm{in}})
\rightarrow
d_{\mathrm{out}}\times d_{\mathrm{in}}
\]
costs
\[
\Theta(d_{\mathrm{out}}d_{\mathrm{in}}r_i).
\]
Summing over all participating clients gives
\[
\Theta\left(d_{\mathrm{out}}d_{\mathrm{in}}
\sum_{i\in\mathcal C_t}r_i\right)
=
\Theta(d_{\mathrm{out}}d_{\mathrm{in}}r_\Sigma).
\]
Since \(d_{\mathrm{out}}\) and \(d_{\mathrm{in}}\) are typically much larger
than the LoRA ranks, this dense-product term dominates the lower-order
factor aggregation and readout costs. Hence,
\[
\mathcal{C}_{\mathrm{HetLoRA}}
=
\Theta(d_{\mathrm{out}}d_{\mathrm{in}}r_\Sigma).
\]

\paragraph{FlexLoRA.}
FlexLoRA first forms the dense semantic update
\[
\bar{\Delta W}^t
=
\sum_{i\in\mathcal C_t}p_i^tB_i^tA_i^t.
\]
Constructing the dense update costs
\[
\sum_{i\in\mathcal C_t}
\Theta(d_{\mathrm{out}}d_{\mathrm{in}}r_i)
=
\Theta(d_{\mathrm{out}}d_{\mathrm{in}}r_\Sigma).
\]
It then performs a thin SVD of
\[
\bar{\Delta W}^t\in\mathbb{R}^{d_{\mathrm{out}}\times d_{\mathrm{in}}},
\]
which costs
\[
\Theta(d_{\mathrm{out}}d_{\mathrm{in}}m),
\qquad
m=\min(d_{\mathrm{out}},d_{\mathrm{in}}).
\]
Therefore, the dominant per-module server cost is
\[
\mathcal{C}_{\mathrm{FlexLoRA}}
=
\Theta\bigl(d_{\mathrm{out}}d_{\mathrm{in}}(r_\Sigma+m)\bigr).
\]
This explains the large runtime overhead of FlexLoRA in Table~\ref{tab:efficiency_comparison}.

\begin{table}[!t]
\centering
\caption{Per-module server-side complexity. The table reports the dominant terms used in the main text.}
\label{tab:server_complexity}
\resizebox{0.5\linewidth}{!}{
\begin{tabular}{l c}
\toprule
Method & Server-side complexity per LoRA module \\
\midrule
FedIT &
\(\Theta\!\left(r_\Sigma(d_{\mathrm{out}}+d_{\mathrm{in}})\right)\) \\
HetLoRA &
\(\Theta\!\left(r_\Sigma d_{\mathrm{out}}d_{\mathrm{in}}\right)\) \\
FlexLoRA &
\(\Theta\!\left(d_{\mathrm{out}}d_{\mathrm{in}}(r_\Sigma+m)\right)\) \\
GLoRA &
\(\Theta\!\left(r_\Sigma^2(d_{\mathrm{out}}+d_{\mathrm{in}})+r_\Sigma^3\right)\) \\
\bottomrule
\end{tabular}
}
\end{table}
\paragraph{GLoRA.}
GLoRA avoids dense-update construction and performs aggregation in a
gauge-aware low-rank reference space. Each client update is first represented
in a gauge-fixed subspace--coordinate form
\[
B_i^t=U_i^tT_i^t,\qquad
\widehat A_i^t=T_i^tA_i^t,
\]
where
\[
U_i^t\in\mathbb{R}^{d_{\mathrm{out}}\times r_i},
\qquad
\widehat A_i^t\in\mathbb{R}^{r_i\times d_{\mathrm{in}}}.
\]
The QR-based gauge fixing is performed on the client side and is not counted
as server-side aggregation cost.

On the server, GLoRA forms the weighted concatenated basis
\[
M^t=
\big[\sqrt{p_i^t}U_i^t\big]_{i\in\mathcal C_t}
\in
\mathbb{R}^{d_{\mathrm{out}}\times r_\Sigma}.
\]
The Gram matrix
\[
G^t=(M^t)^\top M^t
\in\mathbb{R}^{r_\Sigma\times r_\Sigma}
\]
costs
\[
\Theta(d_{\mathrm{out}}r_\Sigma^2),
\]
and its eigendecomposition costs
\[
\Theta(r_\Sigma^3).
\]
The resulting reference basis \(U_{\mathrm{ref}}^t\) is then used to translate
client coordinates into the shared server frame. Under the common setting
\(R=O(r_\Sigma)\), this coordinate aggregation costs
\[
\Theta(r_\Sigma^2 d_{\mathrm{in}}).
\]
Therefore, the dominant per-module server cost of GLoRA is
\[
\mathcal{C}_{\mathrm{GLoRA}}
=
\Theta\!\left(
r_\Sigma^2(d_{\mathrm{out}}+d_{\mathrm{in}})
+
r_\Sigma^3
\right).
\]
This compact form matches the complexity reported in
Table~\ref{tab:efficiency_comparison}. Since \(r_\Sigma\ll d_{\mathrm{out}},d_{\mathrm{in}}\)
in LoRA adaptation, GLoRA avoids the dense-update and full-SVD operations
required by FlexLoRA while still supporting heterogeneous ranks.